\documentclass[hidelinks,onefignum,onetabnum]{siamart251216}

\usepackage{amsfonts,amssymb,mathtools}
\usepackage{graphicx}
\usepackage{bm}
\usepackage{upgreek} 

\newsiamremark{remark}{Remark}
\newsiamthm{claim}{Claim}

\providecommand{\tightlist}{\setlength{\itemsep}{0pt}\setlength{\parskip}{0pt}}

\headers{Avoiding unsafe sets when training with Langevin dynamics}{A. M. Oberman}

\title{Avoiding unsafe sets when training with Langevin Dynamics\thanks{Submitted to the editors July 10, 2026.
  \funding{This research was supported by NSERC, Coefficient Giving, and CIFAR.}}}

\author{Adam M. Oberman\thanks{Department of Mathematics and Statistics, McGill University, Montreal, QC, Canada;
  Mila, Quebec AI Institute; and LawZero
  (\email{adam.oberman@mcgill.ca}).}}

\begin{document}
\maketitle

\begin{abstract}
Training a model with noisy gradient descent can be idealized as overdamped Langevin dynamics, and a natural safety question is to bound the probability $\nu_t(\mathcal{A}_H) = \mathbb{P}(Q_t \in \mathcal{A}_H)$ that the trajectory lies in a designated failure region $\mathcal{A}_H$. We study this for a smooth, strongly convex loss in $d$ dimensions, with $\mathcal{A}_H$ separated from the minimizer by an energy gap. At the end of training, the equilibrium mass $\pi(\mathcal{A}_H)$ is exponentially small in $d$, with a complementary energy-barrier rate when the noise is small. Along the trajectory, a shape-free bound $\nu_t(\mathcal{A}_H) \le \pi(\mathcal{A}_H)(1 + \sqrt{\chi_0^2/\pi(\mathcal{A}_H)}\,e^{-mt})$ shows the in-set probability relaxes to (twice) the static value after a burn-in of order $d$, using only the global spectral gap $m$. A worked Ornstein--Uhlenbeck example shows this burn-in is necessary: an angular slice of the equilibrium shell can transiently swell by a factor exponential in $d$, though its equilibrium mass is tiny. To rule this out we introduce a local relaxation rate, defined through the spectral measure of the region's centered indicator rather than a Dirichlet-form Rayleigh quotient. For geometrically isolated regions this rate exceeds the global one, shrinking the burn-in, and with a maximum-principle ceiling it caps the trajectory probability uniformly in time. Strong convexity sets how fast training relaxes, but the shape of the unsafe set decides whether the trajectory bulges through it on the way to equilibrium.
\end{abstract}



\section{Introduction}\label{sec:introduction}
A model trained by noisy gradient descent can be idealized as a diffusion
on its loss landscape, and a basic safety question is whether the training
trajectory ever enters a designated bad region of parameter space. Write
\(Q_t \in \mathbb{R}^d\) for the parameters at training time \(t\) and
\(\mathcal{A}_H \subseteq \mathbb{R}^d\) for a failure region: a set of
parameters whose induced behavior we would like the trained model to avoid.
Even when training ends safely, the trajectory can pass through
\(\mathcal{A}_H\) on its way to the optimum, so the object of interest is
the in-set probability
\[
\nu_t(\mathcal{A}_H) \;=\; \mathbb{P}(Q_t \in \mathcal{A}_H)
\]
at every training time \(t\), not only at convergence.

Several safety concerns share this shape. In code generation,
\(\mathcal{A}_H\) is the set of parameters that emit a hidden backdoor or a
known-insecure pattern, and one wants the chance that training ever lands
there to be negligible. In alignment, \(\mathcal{A}_H\) is a region of
misaligned or deceptive behavior that a model might drift through before
settling into a benign optimum. The motivating instance for this work is
the Scientist AI (SAI) Predictor safety case of \cite{SAI26}, which
separates an honest non-agentic Predictor from a scaffold that gates its
outputs through a guardrail, and bounds the probability that a
consequence-invariant training process produces a \emph{dangerous}
Predictor (one whose guarded deployment causes a designated harm event
above a normative threshold) uniformly in \(t\) by
\[
\nu_t(\mathcal{A}_H) \;\le\; C_{\mathrm{bad}}\, R_{\mathrm{shell}}.
\]
Here \(R_{\mathrm{shell}}\) is the conditional fraction of dangerous
Predictors inside a narrow loss band under the initialization, argued
exponentially small on the grounds that danger requires many coordinated
errors, and \(C_{\mathrm{bad}}\) is the within-band enrichment factor of
training, assumed bounded as a stated requirement on the process. That
argument treats the dynamics generating \(\nu_t\) abstractly, as a
distribution over training trajectories indexed by \(t\); the present paper
supplies those dynamics and bounds \(\nu_t(\mathcal{A}_H)\) directly.

We model a training run as the overdamped Langevin dynamics
\[
dQ_t = -\nabla J(Q_t)\, dt + \sigma\, dW_t,
\]
where \(J\) is the training loss on \(\mathbb{R}^d\), \(W_t\) is standard
Brownian motion, and \(\sigma > 0\) is a noise level set by the
optimization (informed, for example, by batch size and learning rate).

This paper proves an upper bound on the probability that
the law of a Langevin training trajectory occupies the designated failure
region \(\mathcal{A}_H \subseteq \mathbb{R}^d\). Under explicit
smoothness, convexity, and energy-gap hypotheses on the loss \(J\), the
bound is exponentially small in \(d\) and is uniform in time,  after a
burn-in of order \(d\). Two grades are proved: a shape-free version that
uses only the total equilibrium mass \(\pi(\mathcal{A}_H)\) and the
global Poincaré constant, and a shape-aware version that uses a local
relaxation rate \(\lambda_{\mathcal{A}_H}\) together with a
maximum-principle ceiling to remove the burn-in altogether for
flux-isolated sets.

\subsection*{The Langevin idealization} This SDE is the standard
continuous-time model of stochastic gradient training: at small step
size, minibatch SGD on a smooth loss has Itô-SDE limits with drift
\(-\nabla J\) and noise covariance set by the minibatch gradient
covariance \cite{LTE19, MHB17, HLLL19}, and the explicitly noised SGLD
algorithm of \cite{WT11} realizes the model used here exactly. Isotropy
of the noise and the continuous-time limit are idealizations; we take
them as given. \Cref{sec:discrete-time} records how the set-mass bounds
transfer to the discrete-time algorithm; the anisotropy gap is not
addressed here.

The loss \(J\) has a minimizer \(Q_\star\), normalized to \(J(Q_\star) = 0\). We
single out a \textbf{failure region}
\(\mathcal{A}_H \subseteq \mathbb{R}^d\). The minimizer is safe,
\(Q_\star \notin \mathcal{A}_H\), but the noise keeps \(Q_t\) from ever
settling exactly at \(Q_\star\), so we must ask how much probability mass
the trajectory places in \(\mathcal{A}_H\).

The natural hope is that if the failure region has small equilibrium
mass \(\pi(\mathcal{A}_H)\), then \(\mathbb{P}(Q_t \in \mathcal{A}_H)\)
is small for all \(t\). This hope is false in general, and the way it
fails is the organizing problem of this paper.

\subsection*{Transient swelling} Even under the strongest possible
convexity, the trajectory mass of a small set can bulge far above both
its initial and its equilibrium value. The canonical demonstration is
the one-dimensional Ornstein-Uhlenbeck process \(J(Q) = Q^2/2\) with
\(\pi = N(0,1)\). Start the
trajectory concentrated near \(Q = 10\) and take
\(\mathcal{A}_H = [4,6]\). Then the initial mass of \(\mathcal{A}_H\) is
zero and its equilibrium mass is about \(10^{-5}\), yet at
\(t = \log 2\) the law is approximately \(N(5, 3/4)\) and places roughly
\(0.7\) of its mass in \(\mathcal{A}_H\). The mass swells by a factor of
about \(10^5\) on its way home, because \(\mathcal{A}_H\) sits directly
on the transport path from the start to equilibrium. No convergence rate
for KL or Wasserstein distance forbids this: those are global
functionals and say nothing about a single set.

\begin{figure}[t]
  \centering
  \footnotesize
  \includegraphics[width=\linewidth]{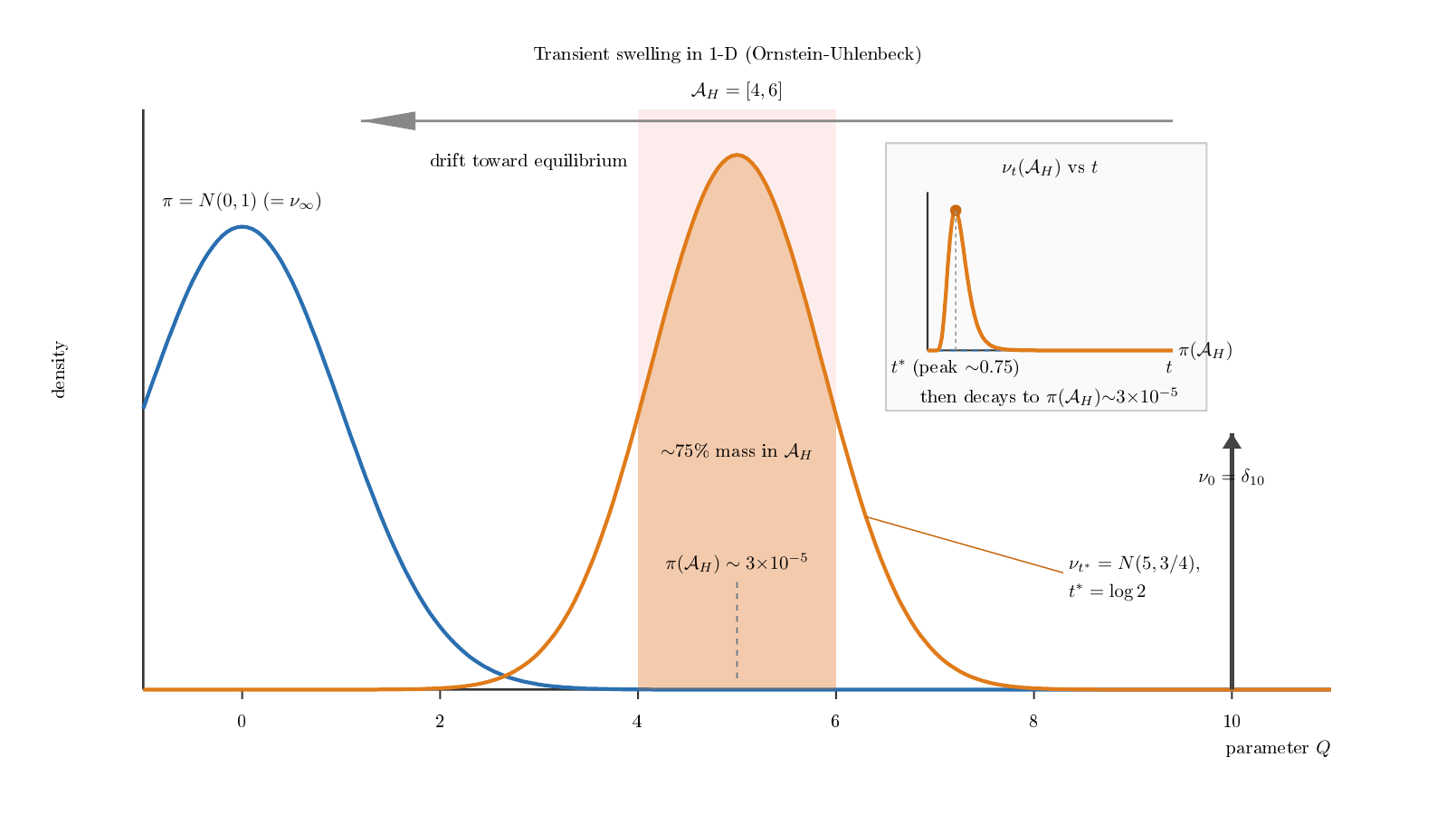}
  \caption{Transient swelling in the 1D Ornstein-Uhlenbeck process with
    $\pi = N(0,1)$. From the point-mass start $\nu_0 = \delta_{10}$, the
    law $\nu_t$ passes through $N(5, 3/4)$ at $t^* = \log 2$, placing
    about $75\%$ of its mass in $\mathcal{A}_H = [4,6]$, whose
    equilibrium mass is $\pi(\mathcal{A}_H) \approx 3 \times 10^{-5}$.
    The inset shows $\nu_t(\mathcal{A}_H)$ peaking near $t^*$ and
    decaying to $\pi(\mathcal{A}_H)$.}
  \label{fig:swelling}
\end{figure}

\subsection*{Higher-dimensional swelling} Swelling is not a one-dimensional artifact. In high dimensions the equilibrium of a strongly convex loss concentrates on a thin shell, and a failure region that is an angular slice of that shell suffers the same transit swelling, now amplified by
dimension. Section 7 makes this quantitative: the transient overshoot
factor for shell dynamics grows exponentially in the ambient dimension
unless the geometry of \(\mathcal{A}_H\) is controlled. This is the
concrete reason we cannot rely on equilibrium mass alone, and it is what
motivates a geometric (Cheeger-type) notion of isolation.

\subsection*{What this paper proves} We give two complementary estimates.

\begin{enumerate}
\def\labelenumi{\arabic{enumi}.}
\tightlist
\item
  A \textbf{static mass bound} on \(\pi(\mathcal{A}_H)\), the
  probability at the end of training, which is exponentially small in
  the dimension \(d\) (Section 4).
\item
  A \textbf{dynamic mass bound} on \(\mathbb{P}(Q_t \in \mathcal{A}_H)\)
  along the trajectory (Section 5), which relaxes to the static bound
  after a burn-in time of order \(d\).
\end{enumerate}

The dynamic bound comes in two grades. The first (Section 5) assumes
only the total equilibrium mass \(\pi(\mathcal{A}_H)\) and gives a
clean, dimension-free relaxation rate, at the cost of a transient window
during which it is uninformative (the swelling window).
The second (Section 6) assumes in addition that \(\mathcal{A}_H\) is
geometrically isolated, quantified by a local relaxation rate
\(\lambda_{\mathcal{A}_H} \ge m\), which shrinks the transient window by
the factor \(m/\lambda_{\mathcal{A}_H}\) and, combined with a
maximum-principle ceiling, caps the trajectory mass uniformly in time.
Section 7 returns to the Ornstein-Uhlenbeck and shell examples to show
both bounds are sharp and to exhibit the geometries where swelling is
real.

\subsection{Notation}\label{sec:notation}

Throughout, \(\nu_t\) denotes the law of \(Q_t\), \(\pi\) the stationary
Gibbs measure, and \(\mathcal{A}_H\) the failure region. We use
\(\sigma\) for the noise level (rather than an inverse temperature
\(\beta\); the two are related by \(\beta = 2/\sigma^2\)). The following quantities
measure how far the initial law \(\nu_0\) is from \(\pi\) and
how isolated \(\mathcal{A}_H\) is:

\begin{itemize}
\tightlist
\item
  \(M := \|\nu_0/\pi\|_\infty\), the initial density ratio in
  \(L^\infty\).
\item
  \(\chi_0^2 := \chi^2(\nu_0 \,\|\, \pi)\), the chi-squared divergence.
\item
  \(\lambda_{\mathcal{A}_H}\), the local spectral gap of
  \(\mathcal{A}_H\) (Definition in Section 6).
\end{itemize}

\section{Related work and positioning}\label{sec:related-work-and-positioning}

\subsection*{The object controlled} Much of the sampling literature \cite{VW19, EHZ22, CELSZ22, Che24, Pav14, BGL14} bounds a global divergence of the
law from the target: KL, chi-squared, Rényi, or Wasserstein. The
metastability literature \cite{MS14, BEGK04, BGK05} bounds spectral gaps
and exit times of metastable wells. The SGLD hitting-time analyses
\cite{RRT17, ZLC17} bound the time to first reach a target region.
Fokker-Planck analyses of the training dynamics itself \cite{DZ21}
characterize which minima are selected, flat versus sharp, as a function
of the noise or batch size, an asymptotic minima-selection question
complementary to the finite-time set occupancy studied here. The
object controlled here is none of these: it is \(\nu_t(\mathcal{A}_H)\),
the probability of being in a fixed measurable set at a fixed time.
No global divergence forbids transient swelling: a global functional
can be small while a single set's mass is large (Section 7).

\subsection*{Structure of the bound} \cref{thm:dynamic-mass} is
downstream of the standard \(L^2(\pi)\) contraction
\(\|u_t - 1\|_{L^2(\pi)} \le \sqrt{\chi_0^2}\,e^{-mt}\), which is in
\cite{Pav14, MV00, AMTU01, BGL14}. The new step is pairing this
contraction with the \emph{centered} indicator
\(\mathbf{1}_{\mathcal{A}_H} - \pi(\mathcal{A}_H)\) rather than the raw
indicator. Centering replaces the second moment \(\pi(\mathcal{A}_H)\)
by the variance \(\pi(\mathcal{A}_H)(1-\pi(\mathcal{A}_H))\), which is
what makes the bound informative for rare sets. The centering device is
the ``warm-start'' trick standard in geometric sampling \cite{LS93, LV07}, applied here to a fixed measurable set in conjunction with the
chi-squared contraction. The contribution is not in any single
ingredient but in the synthesis: a static dimensional bound on
\(\pi(\mathcal{A}_H)\) (Section 4, sublevel volume plus smoothness, in
the spirit of Gaussian concentration but with explicit constants), the
\(L^2\)-to-set-mass conversion (Section 5), and the maximum-principle
ceiling (\cref{prop:max-principle}) that closes the transient
window from above.

\subsection*{Geometric input} The local rate
\(\lambda_{\mathcal{A}_H}\) of Definition~\ref{def:local-rate} is
related to two existing notions. In \cite{ZLC17} a
Cheeger constant of the target region is used to bound the SGLD hitting
time of approximate stationary points; this and
\(\lambda_{\mathcal{A}_H}\) are siblings, both measuring how isolated a
set is under the dynamics, but \cite{ZLC17} uses it on the favorable side
of the dynamics (entering a good set) while \cref{thm:flux-isolation}
uses it on the unfavorable side (avoiding a bad set), so the resulting
inequalities run in opposite directions. In the metastability literature
\cite{MS14, BEGK04, BGK05} the analogous quantity is the Eyring-Kramers
rate of escape from a metastable well, which governs case (b) of the
Section 6.3 taxonomy (far-tail ball) through the Arrhenius-small
slow-mode overlap of \cref{prop:tax-b}. That literature does not define a comparable quantity
for cases (a) and (c), the off-shell ball at the mode and the
shell-slice, since neither is a basin. \cref{def:local-rate}
handles all three uniformly by defining \(\lambda_{\mathcal{A}_H}\)
through the spectral measure of the centered indicator rather than
through a Dirichlet-energy Rayleigh quotient, which is the technical
move that makes the definition work on indicators outside the Dirichlet
domain (Section 6.1).

\subsection*{What is not improved} This paper does not strengthen the global
Poincaré constant under log-concavity; the rate \(m\) from Brascamp-Lieb
\cite{BL76}, from the Bakry-Émery \(\Gamma_2\) derivation \cite{BE85, BGL14}, and, in the non-strongly-log-concave case, from the KLS work
\cite{LV24} and the surveys of Cattiaux-Guillin \cite{CG14}, is used as a
black box. Similarly, the discrete-time Langevin Monte Carlo analyses
\cite{CELSZ22, EHZ22} give convergence under weaker functional
inequalities than the chi-squared bound used here (Poincaré,
Latala-Oleszkiewicz, modified log-Sobolev); the discrete-time transfer
of \cref{sec:discrete-time} builds on them rather than competes with
them.

\subsection*{Framing} The motivation is to bound the probability that a
Langevin training trajectory occupies a designated failure region at any
finite time, in a high-dimensional parameter space. This places the
paper next to \cite{RRT17, ZLC17} in motivation, with a complementary
deliverable: those papers bound when training reaches a good set, while
this paper bounds whether training is currently in a bad set, in a
shape-free form (\cref{thm:dynamic-mass}) and a shape-aware form
(\cref{thm:flux-isolation}), which makes the cost of geometric
ignorance explicit. The OU and shell-slice examples of Section 7 show
the gap between the two forms is real and exponential in \(d\), which
is the strongest case for why a local rate is needed and not just a
faster global rate.

\subsection*{Scope and limitations} The results rest on the following
idealizations. First, the
dynamics is the continuous-time, isotropic-noise Langevin SDE; minibatch
SGD is discrete and anisotropic, and we do not close either gap here,
though \cref{sec:discrete-time} shows the set-mass conversion transfers
to discrete-time Langevin Monte Carlo given any chi-squared convergence
bound for the chain. Second, global strong convexity is a basin-local
idealization: it should be read as modeling a fine-tuning or
late-training phase within one basin of attraction, not the full
nonconvex landscape. The dynamic bounds use convexity only through the
spectral gap of \(\pi\), so they extend as stated to any potential whose
Gibbs measure satisfies a Poincar\'e inequality \cite{CG14, LV24}, at
the price of the explicit constants. Third, the energy-gap link
\(J \ge \psi(a_H)\) is a hypothesis, not a theorem: it asserts that
unsafe parameters are visible to the loss. Failure regions that are
loss-indistinguishable (a deceptive model achieving low loss) violate it
by definition and are outside the reach of any argument that works
through \(\pi\); in the safety case of \cite{SAI26} such regions are the
responsibility of a different layer (the guardrail), not of the training
dynamics. Fourth, the local rate \(\lambda_{\mathcal{A}_H}\) of Section
6 is verified per family of sets (\cref{sec:taxonomy-proofs}) rather
than by a universal geometric criterion; the conductance
\(h_{\mathcal{A}_H}\) is a diagnostic, and we prove no Cheeger-type
comparison between the two.

\section{Setup}\label{sec:setup}

\subsection*{The dynamics and its equilibrium} With
\(J \in C^2(\mathbb{R}^d)\) and minimizer \(Q_\star\), \(J(Q_\star) = 0\), the
stationary distribution of the SDE is the Gibbs measure
\[
\pi(Q) = \frac{1}{Z}\, e^{-2 J(Q)/\sigma^2}, \qquad Z = \int_{\mathbb{R}^d} e^{-2 J(Q)/\sigma^2}\, dQ.
\]
This captures the distribution of \(Q_T\) at the end of training. The
law \(\nu_t\) of \(Q_t\) solves the Fokker-Planck equation

\begin{equation}\label{eq:fokker-planck}
\partial_t \nu_t = \nabla \cdot (\nu_t\, \nabla J) + \frac{\sigma^2}{2}\, \Delta \nu_t, \qquad \nu_t\big|_{t=0} = \nu_0.
\end{equation}

\subsection*{Remark (consequence invariance)} The dynamics depends on the
loss landscape only through \(J\): the drift is \(-\nabla J\), the noise
is independent of everything, and \(\nu_0\) is assumed independent of
any auxiliary objective. Replacing any function the procedure does not
see by an arbitrary other function leaves the trajectory law unchanged.
This rules out by construction the mesa-optimization \cite{Hub19} and
reward-hacking \cite{Ska22} failure modes, in which a training procedure
ostensibly minimizing \(J\) implicitly steers toward a hidden objective.
The contrast is the implicit-bias literature for discrete-time SGD
\cite{BD21, SDBD21}: finite-step discretization introduces implicit
regularizers (such as \(\|\nabla J\|^2/2\)) on top of \(J\), whereas the
continuous-time SDE studied here does not.

\subsection*{Convexity and smoothness} Most results assume the two-sided
Hessian envelope
\[
m\, I \;\preceq\; \nabla^2 J(Q) \;\preceq\; L\, I \qquad \text{for all } Q \in \mathbb{R}^d, \quad 0 < m \le L,
\]
which integrates to the quadratic envelope
\[
\frac{m}{2}\|Q - Q_\star\|^2 \;\le\; J(Q) \;\le\; \frac{L}{2}\|Q - Q_\star\|^2.
\]
The upper bound (\(L\)-smoothness) alone suffices for the static
dimensional bound of Section 4; the lower bound (\(m\)-strong convexity)
is what supplies the spectral gap used in Sections 5 and 6.

\subsection*{The failure region} We take
\[
\mathcal{A}_H = \{Q \in \mathbb{R}^d : a_H(Q) > \alpha\}
\]
for a measurable alarm function \(a_H \ge 0\) with \(a_H(Q_\star) = 0\), and
we assume a monotone link to the loss,
\[
J(Q) \;\ge\; \psi(a_H(Q)),
\]
for a strictly increasing \(\psi: [0,\infty) \to [0,\infty)\) with
\(\psi(0) = 0\). By monotonicity
\(\mathcal{A}_H \subseteq \{J \ge \psi(\alpha)\}\): triggering the alarm
costs at least \(\psi(\alpha)\) in loss. This energy gap is what makes
\(\mathcal{A}_H\) rare under \(\pi\).

\subsection*{Chi-squared divergence of the start} When the initial law
matters we record
\[
\chi_0^2 \;:=\; \chi^2(\nu_0 \,\|\, \pi) \;=\; \int \left(\frac{\nu_0}{\pi} - 1\right)^2 d\pi \;=\; \int \frac{\nu_0^2}{\pi}\, dQ - 1.
\]
\section{Static mass bound: probability at the end of training}\label{sec:static-mass-bound-probability-at-the-end-of-training}

The stationary measure \(\pi\) describes \(Q_T\) for \(T\) large, so
\(\pi(\mathcal{A}_H)\) is the probability of ending training inside the
failure region. We show it is exponentially small in \(d\). The
mechanism is dimensional: \(\pi\) spreads over a large volume, and the
rare, high-loss set \(\mathcal{A}_H\) captures exponentially little of
it. This section uses only \(L\)-smoothness and a volume bound on
\(\mathcal{A}_H\).

\subsection*{Dimensionless variables} Define the natural length scale, the
dimensionless energy variable, and the dimensionless gap
\[
\ell^2 \;:=\; \frac{\uppi\sigma^2}{L}, \qquad s \;:=\; \frac{2 v}{\sigma^2}, \qquad \hat\alpha \;:=\; \frac{2 \psi(\alpha)}{\sigma^2}.
\]
\subsection*{Sublevel-volume bound} Let
\(V(v) := \big|\{Q \in \mathcal{A}_H : J(Q) \le v\}\big|\) be the
Lebesgue volume of the part of \(\mathcal{A}_H\) below loss level \(v\).
Assume \(V(v) \le \Phi(v)\) for some non-decreasing \(\Phi\) with
\(\Phi(v) = 0\) for \(v < \psi(\alpha)\), and set the dimensionless
sublevel volume
\[
\hat\Phi(s) \;:=\; \Phi(s\sigma^2/2)\,\big/\,\ell^d,
\]
normalized by the volume cell \(\ell^d\). It satisfies
\(\hat\Phi(s) = 0\) for \(s < \hat\alpha\).

\begin{proposition}[dimensionless static mass bound]\label{prop:static-mass}
Under \(\nabla^2 J \preceq L\,I\) and
\(J(Q_\star) = 0\),
\[
\pi(\mathcal{A}_H) \;\le\; \int_0^\infty \hat\Phi(s)\, e^{-s}\, ds \;=\; \mathbb{E}_{S \sim \mathrm{Exp}(1)}\, \hat\Phi(S).
\]
\end{proposition}

\begin{proof}
Write \(\pi(\mathcal{A}_H) = N/Z\). The layer-cake
identity
\(e^{-2 J/\sigma^2} = (2/\sigma^2) \int_0^\infty \mathbf{1}_{\{J \le v\}}\, e^{-2 v/\sigma^2}\, dv\)
and Fubini give
\[
N \;=\; \frac{2}{\sigma^2} \int_0^\infty V(v)\, e^{-2 v/\sigma^2}\, dv \;\le\; \frac{2}{\sigma^2} \int_0^\infty \Phi(v)\, e^{-2 v/\sigma^2}\, dv \;=\; \ell^d \int_0^\infty \hat\Phi(s)\, e^{-s}\, ds,
\]
the last equality being the substitution \(s = 2v/\sigma^2\). The
smoothness envelope \(J(Q) \le \tfrac{L}{2}\|Q - Q_\star\|^2\) gives
\(e^{-2 J/\sigma^2} \ge e^{-L\|Q-Q_\star\|^2/\sigma^2}\), so
\[
Z \;\ge\; \int_{\mathbb{R}^d} e^{-L\|Q-Q_\star\|^2/\sigma^2}\, dQ \;=\; \left(\frac{\uppi\sigma^2}{L}\right)^{d/2} \;=\; \ell^d.
\]
Divide.
\end{proof}

\begin{theorem}[exponential-in-\(d\) static mass bound]\label{thm:static-mass}
Suppose the sublevel volume has the canonical
polynomial form \(\Phi(v) = C\,(v - \psi(\alpha))_+^\eta\) for some
\(C > 0\), \(\eta > 0\). Define the dimensionless rate and prefactor
\[
A \;:=\; \log\!\left(\frac{\uppi\sigma^2}{L}\right), \qquad K \;:=\; C\,\Gamma(\eta + 1)\,\left(\frac{\sigma^2}{2}\right)^\eta e^{-\hat\alpha},
\]
with \(\hat\alpha = 2\psi(\alpha)/\sigma^2\). When \(\sigma\) is large
enough that \(\uppi\sigma^2 > L\) (equivalently \(A > 0\)),
\[
\pi(\mathcal{A}_H) \;\le\; K\, e^{-A d/2}.
\]
For fixed \(\sigma\), \(L\), and \(\psi(\alpha)\), the static mass
decays exponentially in the ambient dimension \(d\) at rate \(A/2 > 0\),
with an additional energy-gap suppression
\(e^{-\hat\alpha} = e^{-2\psi(\alpha)/\sigma^2}\) in the prefactor.
\end{theorem}

\begin{proof}
Apply \cref{prop:static-mass} with
\(\hat\Phi(s) = \hat C\,(s - \hat\alpha)_+^\eta\) and substitute
\(u = s - \hat\alpha\):
\[
\int_0^\infty \hat C\,(s - \hat\alpha)_+^\eta\, e^{-s}\, ds \;=\; \hat C\, e^{-\hat\alpha} \int_0^\infty u^\eta\, e^{-u}\, du \;=\; \hat C\, \Gamma(\eta + 1)\, e^{-\hat\alpha}.
\]
From the dimensionless variables
\(\hat C = C\,(\sigma^2/2)^\eta / \ell^d\) and
\(\ell^d = (\uppi\sigma^2/L)^{d/2} = e^{A d/2}\). Substituting gives
\(\pi(\mathcal{A}_H) \le K\, e^{-A d/2}\).
\end{proof}

\begin{remark}[what is and is not dimension-free]\label{rem:d-dependence}
The decay rate \(A/2\) in \cref{thm:static-mass} is dimension-free, but
the prefactor \(K\) inherits any dimension dependence carried by the
sublevel constant \(C\). The bound is exponential in \(d\) for
families \(\mathcal{A}_H\) whose sublevel data \((C, \eta)\) are \(O(1)\)
in \(d\): failure regions cut out by a bounded number of coordinates, or
a fixed angular fraction of the equilibrium shell, whose geometry, and
hence \(C\) and \(\eta\), does not change with the ambient dimension.
When \(C = C(d)\) grows with \(d\) (for instance a full-dimensional slab),
that growth must be set against \(e^{-A d/2}\), and the statement is read
at fixed geometry with the dimensional accounting carried explicitly by
\(C(d)\). The same caveat applies to the combined bound below.
\end{remark}

\begin{figure}[t]
  \centering
  \footnotesize
  \includegraphics[width=\linewidth]{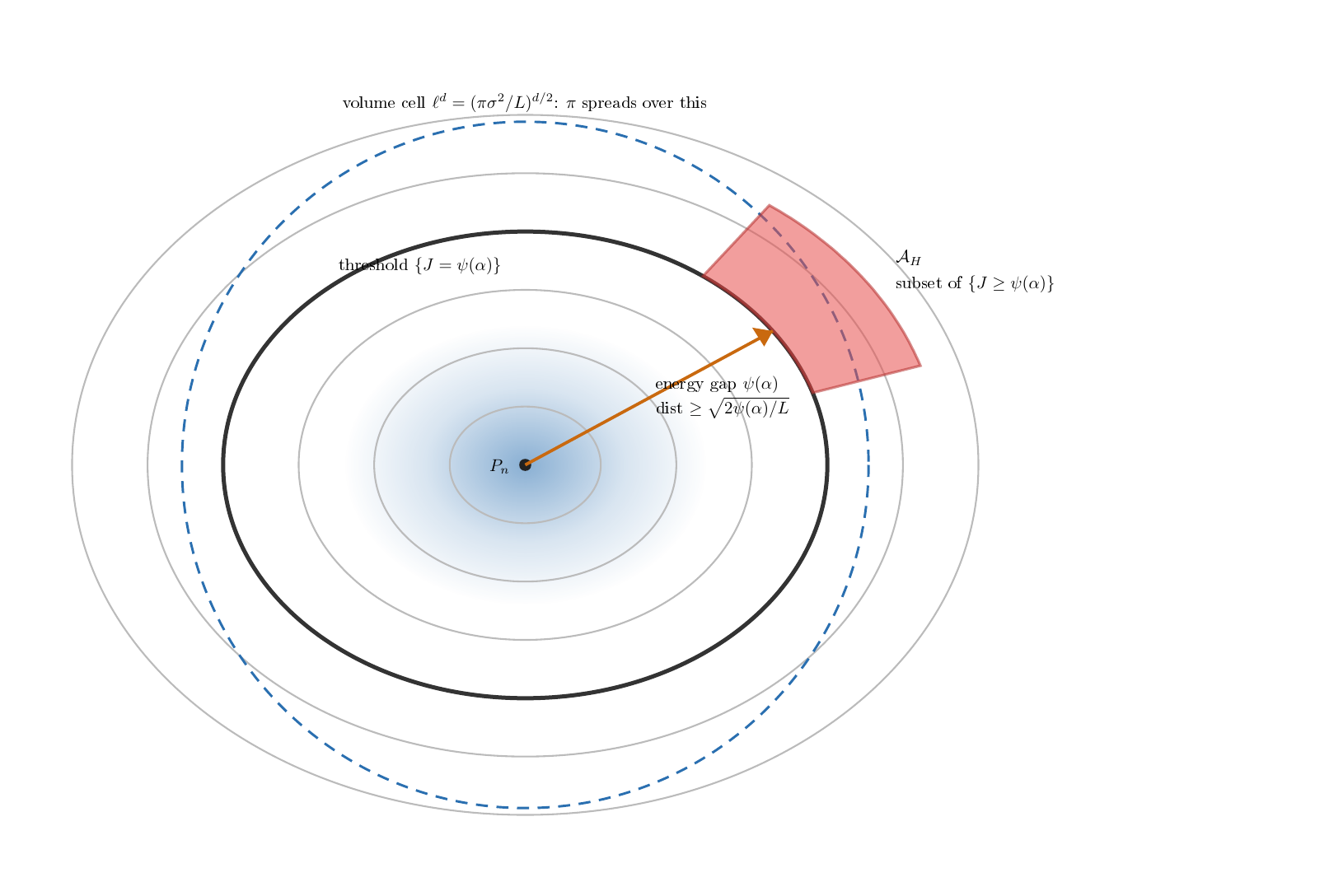}
  \caption{The level-set cap. The equilibrium $\pi$ spreads over the
    volume cell $\ell^d = (\uppi\sigma^2/L)^{d/2}$ (dashed circle), while
    the failure region $\mathcal{A}_H$ sits past the threshold
    $\{J = \psi(\alpha)\}$ at distance $\ge \sqrt{2\psi(\alpha)/L}$ from
    $Q_\star$. Both effects suppress $\pi(\mathcal{A}_H)$: volume spreading
    captures a factor $e^{-Ad/2}$, the energy gap contributes the
    Arrhenius factor $e^{-2\psi/\sigma^2}$.}
  \label{fig:cap}
\end{figure}

\subsection{The small-\(\sigma\) (Arrhenius) regime and the combined bound}\label{sec:the-small-sigma-arrhenius-regime-and-the-combined-bound}

\cref{thm:static-mass} is informative only when
\(\uppi\sigma^2 > L\), that is, when the noise is large or the dimension
high. In the opposite regime a complementary barrier argument applies,
using strong convexity. Since \(\pi\) is \((2m/\sigma^2)\)-strongly
log-concave it has Gaussian tails of width \(\sigma/\sqrt{m}\) around
\(Q_\star\), while smoothness places \(\mathcal{A}_H\) at radius at least
\(r_0 := \sqrt{2\psi(\alpha)/L}\) from \(Q_\star\). The map
\(Q \mapsto \|Q - Q_\star\|\) is \(1\)-Lipschitz with
\(\mathbb{E}_\pi\|Q - Q_\star\| \le \sqrt{d\sigma^2/(2m)}\) (Brascamp-Lieb
bounds the covariance of \(\pi\) by \((\sigma^2/2m)\,I\)), so the Herbst
concentration bound for the \((2m/\sigma^2)\)-log-Sobolev measure \(\pi\)
\cite[Ch.~5]{BGL14} gives the explicit estimate
\[
\pi(\mathcal{A}_H) \;\le\; \exp\!\left(-\frac{m}{\sigma^2}\Big(r_0 - \sqrt{\tfrac{d\sigma^2}{2m}}\Big)_{\!+}^{2}\right).
\]
When the barrier dominates the dimension, \(\psi(\alpha) > dL\sigma^2/(4m)\)
(equivalently \(r_0 > \sqrt{d\sigma^2/(2m)}\)), this is the bare Arrhenius
rate \(2\psi(\alpha)/\sigma^2\) shaved by the condition number
\(m/L \in (0,1]\),
\[
\pi(\mathcal{A}_H) \;\le\; \exp\!\left(-\frac{2(m/L)\,\psi(\alpha)}{\sigma^2}\Big(1 - \sqrt{\tfrac{dL\sigma^2}{4m\psi(\alpha)}}\,\Big)^{2}\right),
\]
the dimensional correction in parentheses tending to \(1\) as
\(\psi(\alpha)\big/\big(dL\sigma^2/(4m)\big) \to \infty\). Folding the two bounds together
through \(\min(e^{-a}, e^{-b}) = e^{-\max(a,b)}\) and defining the
\textbf{effective dimensional rate}
\[
A_{\mathrm{eff}} \;:=\; \max\!\left(A,\;\; \frac{4m\,\psi(\alpha)}{L\,d\,\sigma^2}\right),
\]
gives the \textbf{combined static bound}
\[
\pi(\mathcal{A}_H) \leq e^{-A_{\mathrm{eff}}\, d/2}.
\]
In the large-\(\sigma\) or large-\(d\) regime, \(A_{\mathrm{eff}} = A\)
and spreading does the work. In the small-\(\sigma\) regime
\(\sigma^2 \ll 4m\psi(\alpha)/(dL)\),
\(A_{\mathrm{eff}}\, d/2 = 2m\psi(\alpha)/(L\sigma^2)\) and the energy
barrier does the work. The two regimes meet at the crossover
\(Ad/2 \approx 2m\psi(\alpha)/(L\sigma^2)\), where the bound is least
tight, and this crossover coincides with the burn-in scale of Section 5.

\section{Dynamic mass control along the trajectory}\label{sec:dynamic-mass-control-along-the-trajectory}

The static bound is a statement at stationarity (\(t \to \infty\)). For
finite \(t\) we control
\(\nu_t(\mathcal{A}_H) = \mathbb{P}(Q_t \in \mathcal{A}_H)\) directly.
The main result of this section assumes only the total mass
\(\pi(\mathcal{A}_H)\), together with strong convexity and a finite
initial chi-squared divergence. It is the dimension-free, shape-free
bound. Its one weakness, a transient window in which it is
uninformative, is the swelling phenomenon of the introduction,
and it is shortened in Section 6 by adding geometric information.

\subsection{The common machinery}\label{sec:the-common-machinery}

All bounds in this section and the next track the density ratio
\[
u_t := \frac{\nu_t}{\pi}.
\]
Substituting \(\nu_t = u_t \pi\) into \cref{eq:fokker-planck} and using
the identity \(\nabla \pi = -\frac{2}{\sigma^2}\, \pi\, \nabla J\) (the
chain rule applied to \(\pi \propto e^{-2J/\sigma^2}\)), the
time-independent factors cancel and \(u_t\) solves the backward
Kolmogorov equation
\[
\partial_t u_t = \mathcal{L} u_t, \qquad \mathcal{L} = \frac{\sigma^2}{2}\,\Delta - \nabla J \cdot \nabla,
\]
where \(\mathcal{L}\) is the Langevin generator. Integration by parts
against \(\pi\) shows \(\mathcal{L}\) is self-adjoint in \(L^2(\pi)\)
with Dirichlet form
\[
\mathcal{E}(f) \;=\; -\langle f, \mathcal{L} f\rangle_\pi \;=\; \frac{\sigma^2}{2}\int |\nabla f|^2\, d\pi,
\]
and generates a conservative Markov semigroup \(P_t = e^{t\mathcal{L}}\)
with \(P_t 1 = 1\). These are the standard generator, reversibility, and
Dirichlet-form facts for the Langevin diffusion; see Pavliotis \cite[Ch.~4]{Pav14} and \cite[Sec.~1.6 and 4.2]{BGL14}. Every bound below converts a
statement about \(u_t - 1\) into a statement about the set mass by
pairing against the indicator \(\mathbf{1}_{\mathcal{A}_H}\).

\subsection{The chi-squared spectral-gap bound}\label{sec:the-chi-squared-spectral-gap-bound}

\begin{theorem}[mass control along the trajectory]\label{thm:dynamic-mass}
Assume \(m\,I \preceq \nabla^2 J\) and let \(\nu_0\)
have \(\chi_0^2 < \infty\). Then for every measurable
\(\mathcal{A}_H \subseteq \mathbb{R}^d\) and every \(t \ge 0\),

\begin{equation}\label{eq:star}
\mathbb{P}(Q_t \in \mathcal{A}_H) \;\le\; \pi(\mathcal{A}_H)\,\Big(1 + \sqrt{\chi_0^2/\pi(\mathcal{A}_H)}\; e^{-mt}\Big).
\end{equation}

The same argument gives the symmetric lower bound
\(\nu_t(\mathcal{A}_H) \ge \pi(\mathcal{A}_H) - \sqrt{\chi_0^2\,\pi(\mathcal{A}_H)}\,e^{-mt}\),
so \(\nu_t(\mathcal{A}_H) \to \pi(\mathcal{A}_H)\) at exponential rate
\(m\).
\end{theorem}

\begin{proof}
Steps 1 to 3 are the standard machinery of reversible diffusions, for
which we refer to Pavliotis \cite[Ch.~4]{Pav14}; Steps 4 to 6 are the
explicit estimate.

\emph{Step 1 (change of variables).} As recorded in Section 5.1,
\(u_t = \nu_t/\pi\) solves the backward equation
\(\partial_t u_t = \mathcal{L} u_t\), obtained by substituting
\(\nu_t = u_t\pi\) into \cref{eq:fokker-planck} and cancelling the
stationary part using \(\nabla\pi = -\frac{2}{\sigma^2}\pi\nabla J\)
(\cite[Ch.~4]{Pav14}).

\emph{Step 2 (self-adjointness).} Integration by parts against \(\pi\)
makes \(\mathcal{L}\) self-adjoint in \(L^2(\pi)\) with Dirichlet form
\(\mathcal{E}(f) = \frac{\sigma^2}{2}\int|\nabla f|^2\,d\pi\), and
\(P_t = e^{t\mathcal{L}}\) is a Markov semigroup with \(P_t 1 = 1\)
(Section 5.1; \cite[Ch.~4]{Pav14}).

\emph{Step 3 (Poincaré inequality).} Brascamp-Lieb applied to
\(\pi \propto e^{-2J/\sigma^2}\), whose potential has Hessian
\(\nabla^2(2J/\sigma^2) \succeq \frac{2m}{\sigma^2}\,I\), gives
\[
\operatorname{Var}_\pi(f) \;\le\; \frac{1}{m}\,\mathcal{E}(f),
\]
so the spectral gap of \(-\mathcal{L}\) on \(L^2(\pi)\) is at least
\(m\) (see \cite{BL76}, the Bakry-Émery \(\Gamma_2\) derivation
\cite{BE85}, \cite[Sec.~4.8-4.9]{BGL14}, or the textbook statement
\cite[Ch.~4]{Pav14}). The two factors of \(2/\sigma^2\), one from the
log-concavity constant and one from the Dirichlet form, cancel, which is
why the rate is \(m\) and not \(2m/\sigma^2\).

\emph{Step 4 (\(L^2\) contraction).} For any centered
\(g \in L^2(\pi)\), differentiate the squared norm along the semigroup
and use, in order, self-adjointness, the Dirichlet form, and Step 3 (the
centering \(\int P_t g\,d\pi = \int g\,d\pi = 0\) is preserved, so
Poincaré applies to \(P_t g\)):
\[
\frac{d}{dt}\,\|P_t g\|_{L^2(\pi)}^2 \;=\; 2\langle P_t g,\, \mathcal{L} P_t g\rangle_\pi \;=\; -2\,\mathcal{E}(P_t g) \;\le\; -2m\,\|P_t g\|_{L^2(\pi)}^2.
\]
Grönwall's inequality gives
\(\|P_t g\|_{L^2(\pi)} \le e^{-mt}\|g\|_{L^2(\pi)}\).

\emph{Step 5 (apply to \(u_0 - 1\)).} The constant \(1\) is fixed by
\(\mathcal{L}\), so \(u_t - 1 = P_t(u_0 - 1)\). It is centered,
\(\int(u_0 - 1)\,d\pi = \int \nu_0\,dQ - 1 = 0\), with
\(\|u_0 - 1\|_{L^2(\pi)}^2 = \int(\nu_0/\pi - 1)^2\,d\pi = \chi_0^2\).
Step 4 gives

\begin{equation}\label{eq:contraction}
\|u_t - 1\|_{L^2(\pi)} \;\le\; e^{-mt}\,\|u_0 - 1\|_{L^2(\pi)} \;=\; \sqrt{\chi_0^2}\; e^{-mt}.
\end{equation}

\emph{Step 6 (from \(L^2\) decay to set mass).} Since
\(\int(u_t - 1)\,d\pi = 0\) we may subtract the constant
\(\pi(\mathcal{A}_H)\) inside the pairing for free, replacing the
indicator by the centered indicator:
\[
\nu_t(\mathcal{A}_H) - \pi(\mathcal{A}_H) \;=\; \int_{\mathcal{A}_H}(u_t - 1)\, d\pi \;=\; \langle u_t - 1,\; \mathbf{1}_{\mathcal{A}_H} - \pi(\mathcal{A}_H)\rangle_\pi.
\]
Centering tightens Cauchy-Schwarz: the second factor becomes the
variance of a Bernoulli\((\pi(\mathcal{A}_H))\) variable,
\(\|\mathbf{1}_{\mathcal{A}_H} - \pi(\mathcal{A}_H)\|_{L^2(\pi)}^2 = \pi(\mathcal{A}_H)(1-\pi(\mathcal{A}_H)) \le \pi(\mathcal{A}_H)\),
rather than the second moment \(\pi(\mathcal{A}_H)\), which for a rare
set is dramatically smaller. With \cref{eq:contraction},
\[
\big|\nu_t(\mathcal{A}_H) - \pi(\mathcal{A}_H)\big| \;\le\; \|u_t - 1\|_{L^2(\pi)}\,\sqrt{\pi(\mathcal{A}_H)(1-\pi(\mathcal{A}_H))} \;\le\; \sqrt{\chi_0^2\,\pi(\mathcal{A}_H)}\; e^{-mt}.
\]
Dividing the upper inequality by \(\pi(\mathcal{A}_H)\) yields
\cref{eq:star}; keeping the sign gives the symmetric lower bound.
\end{proof}

The centering in Step 6 is the ``warm-start'' device (\cite{LS93},
\cite{LV07}, \cite{VW19}, \cite{CELSZ22}). The analytic-PDE route to the same \(L^2\) decay is in
\cite{MV00}, \cite{AMTU01}, and a sampling-oriented exposition is in
\cite{Che24}.

\begin{remark}[beyond strong convexity: a Poincar\'e inequality suffices]\label{rem:poincare-only}
Strong convexity enters \cref{thm:dynamic-mass}, and
\cref{thm:flux-isolation} below, only through the spectral gap: the
constant \(m\) is used nowhere except as the Poincar\'e constant of
\(\pi\) in Step 3. Both bounds therefore hold verbatim for any potential
\(J\) whose Gibbs measure obeys a Poincar\'e inequality
\(\operatorname{Var}_\pi(f) \le \lambda_P^{-1}\,\mathcal{E}(f)\), with
\(m\) replaced by \(\lambda_P\) throughout, including the burn-in time
and the local-rate speedup of \cref{thm:flux-isolation}. Standard
devices produce such a \(\lambda_P\) without global strong convexity. If
\(J = J_0 + b\) with \(J_0\) strongly convex and \(b\) bounded, the
Holley-Stroock perturbation lemma \cite{HS87} preserves the inequality
with \(\lambda_P \ge m\, e^{-\operatorname{osc}(2b/\sigma^2)}\), where
\(\operatorname{osc}(g) = \sup g - \inf g\); this covers potentials that
are strongly convex up to a bounded, possibly nonconvex, well. If \(J\)
is only convex at infinity (strongly convex outside a compact set), a
Lyapunov or drift-condition argument yields a positive \(\lambda_P\) as
well \cite{CG14}. The static bound of \cref{thm:static-mass} is the one
result that leans on smoothness and strong log-concavity more
essentially; under a bounded perturbation \(b\) its prefactor picks up
the same factor \(e^{\operatorname{osc}(2b/\sigma^2)}\).
\end{remark}

\subsection*{Explicit \(\chi_0^2\) for a Gaussian start} The constant
\(\chi_0^2\) is finite under mild conditions. If
\(\nu_0 = N(Q_\star, \sigma_0^2 I)\) and
\(m I \preceq \nabla^2 J \preceq L I\), then writing the requirement
that the start be no wider than \(\pi\) allows, one has the closed-form
envelope
\[
\chi_0^2 + 1 \;\le\; \left(\frac{\sigma_0^2}{s_-^2}\right)^{-d/2}\!\!\left(2 - \frac{\sigma_0^2}{s_+^2}\right)^{-d/2}, \qquad s_-^2 := \frac{\sigma^2}{2m},\;\; s_+^2 := \frac{\sigma^2}{2L},
\]
finite if and only if \(\sigma_0^2 < 2 s_+^2 = \sigma^2/L\). In
particular \(\log \chi_0^2 = O(d)\), so \(\chi_0^2\) contributes at most
an \(O(d)\) term to the burn-in below. (The clean Gaussian-Gaussian case
\(J = \tfrac{a}{2}\|Q-Q_\star\|^2\) gives equality, and \(\chi_0^2 = 0\)
when \(\sigma_0^2 = \sigma^2/(2a)\), i.e.~the start equals \(\pi\).)

\subsection{The maximum-principle ceiling}\label{sec:the-maximum-principle-ceiling}

If the initial law has a bounded density ratio, a second, purely
uniform-in-time bound is available, and it complements \cref{eq:star} at
small \(t\).

\begin{proposition}[maximum-principle ceiling]\label{prop:max-principle}
If \(M := \|\nu_0/\pi\|_\infty < \infty\), then
for every measurable \(\mathcal{A}_H\) and all \(t \ge 0\),
\[
\nu_t(\mathcal{A}_H) \;\le\; M\, \pi(\mathcal{A}_H).
\]
\end{proposition}

\begin{proof}
Since \(P_t\) is positivity-preserving with
\(P_t 1 = 1\), applying it to \(u_0 \ge 0\) gives
\(0 \le u_t = P_t u_0 \le \|u_0\|_\infty = M\) pointwise \(\pi\)-a.e.
Integrating against \(\pi\) over \(\mathcal{A}_H\),
\(\nu_t(\mathcal{A}_H) = \int_{\mathcal{A}_H} u_t\, d\pi \le M\,\pi(\mathcal{A}_H)\).
\end{proof}

Combining \cref{prop:max-principle} with
\cref{thm:dynamic-mass}, and using
\(\|u_0 - 1\|_{L^2(\pi)}^2 \le \|u_0\|_\infty \int u_0\, d\pi - 1 = M - 1\),
gives the two-sided sandwich with
\(\delta_t := \sqrt{(M-1)\,\pi(\mathcal{A}_H)(1-\pi(\mathcal{A}_H))}\,e^{-mt}\),
\[
\pi(\mathcal{A}_H) - \delta_t \;\le\; \nu_t(\mathcal{A}_H) \;\le\; \min\!\Big(M\,\pi(\mathcal{A}_H),\;\; \pi(\mathcal{A}_H) + \delta_t\Big).
\]
The two upper bounds cross at
\(t_\star = \frac{1}{2m}\log\frac{1-\pi(\mathcal{A}_H)}{(M-1)\pi(\mathcal{A}_H)}\):
the ceiling \(M\pi(\mathcal{A}_H)\) is tighter for \(t \le t_\star\),
the relaxation bound for \(t \ge t_\star\). The ceiling requires
\(M < \infty\), which fails for a point-mass start, and this failure is
not cosmetic: it is the case where swelling is real (Section 7).

\subsection{Burn-in time and the combined safety bound}\label{sec:burn-in-time-and-the-combined-safety-bound}

The bound \cref{eq:star} is informative once its transient term is at
most order one:
\[
\sqrt{\chi_0^2/\pi(\mathcal{A}_H)}\; e^{-mt} \;\le\; 1 \quad\Longleftrightarrow\quad t \;\ge\; \frac{1}{2m}\log\frac{\chi_0^2}{\pi(\mathcal{A}_H)}.
\]
Below this threshold the bound can exceed \(1\) and is vacuous; this is
the swelling window. Substituting the static bound
\(\pi(\mathcal{A}_H) \le K\, e^{-A d/2}\) from
\cref{thm:static-mass} gives the explicit burn-in time
\[
t_\star \;:=\; \frac{1}{2m}\log\frac{\chi_0^2}{K} \;+\; \frac{A}{4m}\, d.
\]
So \(t_\star\) grows linearly in \(d\) at rate \(A/(4m)\), plus a
\(d\)-independent offset set by \(\chi_0^2\) and \(K\). For
\(t \ge t_\star\), combining \cref{eq:star} with the static bound gives the {combined safety bound}
\[
\mathbb{P}(Q_t \in \mathcal{A}_H) \;\le\; 2\,\pi(\mathcal{A}_H) \;\le\; 2 K\, e^{-A d/2},
\]
exponentially small in \(d\) at rate \(A/2\). In words: after a burn-in
of order \(d\), the Langevin trajectory inherits the dimensionally
suppressed safety of the stationary law. This burn-in is
\(\Theta(d)\) in the worst case, and
\cref{sec:sharpness-and-worked-example-ornstein-uhlenbeck-and-the-shell}
shows the transient overshoot inside the window can reach
\(e^{\Theta(d)}\). That is unavoidable for the shape-free bound; it is
what the local rate \(\lambda_{\mathcal{A}_H}\) of
\cref{sec:geometric-isolation-shortening-the-burn-in}, which shrinks the
window by \(m/\lambda_{\mathcal{A}_H}\), and the maximum-principle
ceiling \cref{prop:max-principle}, which removes it entirely for a start
with bounded density ratio, are there to control.

\subsection{From continuous to discrete time}\label{sec:discrete-time}

Step 6 of the proof of \cref{thm:dynamic-mass} is a statement about a
fixed pair of measures, not about the semigroup, so it transfers
verbatim to discrete time.

\begin{lemma}[set-mass conversion]\label{lem:conversion}
For any probability measure \(\nu \ll \pi\) and any measurable
\(\mathcal{A}_H\),
\[
\big|\nu(\mathcal{A}_H) - \pi(\mathcal{A}_H)\big| \;\le\; \sqrt{\chi^2(\nu \,\|\, \pi)\; \pi(\mathcal{A}_H)\,\big(1 - \pi(\mathcal{A}_H)\big)}.
\]
\end{lemma}

\begin{proof}
As in Step 6:
\(\nu(\mathcal{A}_H) - \pi(\mathcal{A}_H) = \langle \nu/\pi - 1,\; \mathbf{1}_{\mathcal{A}_H} - \pi(\mathcal{A}_H)\rangle_\pi\),
and Cauchy-Schwarz bounds the pairing by
\(\sqrt{\chi^2(\nu\,\|\,\pi)}\) times the Bernoulli standard deviation.
\end{proof}

\begin{corollary}[discrete-time mass control]\label{cor:discrete}
Let \(\nu_k\) be the law of the \(k\)-th iterate of the Langevin Monte
Carlo algorithm
\(Q_{k+1} = Q_k - h\,\nabla J(Q_k) + \sqrt{h}\,\sigma\,\xi_k\) with
\(\xi_k \sim N(0, I)\) i.i.d., and suppose the chain satisfies a
chi-squared (equivalently R\'enyi-2) convergence bound
\(\chi^2(\nu_k \,\|\, \pi) \le \varepsilon^2(k)\). Then for every
measurable \(\mathcal{A}_H\) and every \(k \ge 0\),
\[
\nu_k(\mathcal{A}_H) \;\le\; \pi(\mathcal{A}_H) + \varepsilon(k)\,\sqrt{\pi(\mathcal{A}_H)}.
\]
\end{corollary}

Bounds of this form are available under the standing
assumptions: for \(m I \preceq \nabla^2 J \preceq L I\) and step size
\(h \lesssim 1/L\), the chain contracts in R\'enyi-2 divergence at a
geometric rate of order \(mh\) per step, up to an additive bias floor
that vanishes as \(h \to 0\) \cite{VW19, EHZ22, CELSZ22}; see
\cite{Che24} for an exposition. With any such bound, the burn-in
analysis of Section 5.3 carries over with \(mt\) replaced by the
discrete exponent and with \(\pi(\mathcal{A}_H)\) augmented by the bias
floor.

Langevin Monte Carlo does not, however, leave \(\pi\)
invariant, so the maximum-principle ceiling
(\cref{prop:max-principle}) does not transfer: that argument uses
stationarity of \(\pi\) under the semigroup. For the
Metropolis-adjusted algorithm (MALA), which is \(\pi\)-reversible, both
the ceiling and the spectral-gap contraction hold verbatim, with the
spectral gap of the transition kernel in place of \(m\). The
local rate of Section 6 also does not transfer as stated: it is defined
through the spectral measure of the continuous generator, and its
discrete analogue (the spectral measure of
the centered indicator under the kernel) makes sense for MALA but is
not developed here.

\section{Geometric isolation: shortening the burn-in}\label{sec:geometric-isolation-shortening-the-burn-in}

\cref{thm:dynamic-mass} uses the global Poincaré constant \(m\)
and only the total mass \(\pi(\mathcal{A}_H)\). It pays for this
generality with the burn-in window, during which it cannot exclude
swelling. The window is genuine for sets that lie on the transport path
from \(\nu_0\) to \(\pi\) (Section 7), but for many failure regions of
interest it is an artifact of discarding the geometry of
\(\mathcal{A}_H\). This section restores that geometry through a local
relaxation rate \(\lambda_{\mathcal{A}_H} \ge m\). The resulting bound
relaxes at this faster, set-dependent rate, shortening the burn-in by
the factor \(m/\lambda_{\mathcal{A}_H}\), and combined with the
maximum-principle ceiling it caps the trajectory mass uniformly in time.
\Cref{sec:taxonomy-proofs} proves the resulting three-case taxonomy
exactly for the quadratic prototype.

\subsection{A pinned-down local relaxation rate}\label{sec:a-pinned-down-local-relaxation-rate}

The informal idea is that mass cannot enter \(\mathcal{A}_H\) faster
than it can cross the boundary \(\partial\mathcal{A}_H\), so a
well-isolated set should relax faster than the global rate \(m\). Making
this precise requires care, because the obvious candidate fails: the
Rayleigh quotient of the centered indicator
\(\phi_{\mathcal{A}_H} := \mathbf{1}_{\mathcal{A}_H} - \pi(\mathcal{A}_H)\)
is
\[
\frac{\mathcal{E}(\phi_{\mathcal{A}_H})}{\operatorname{Var}_\pi(\phi_{\mathcal{A}_H})}, \qquad \mathcal{E}(\phi_{\mathcal{A}_H}) = \frac{\sigma^2}{2}\int |\nabla \mathbf{1}_{\mathcal{A}_H}|^2\, d\pi = +\infty,
\]
since \(\nabla \mathbf{1}_{\mathcal{A}_H}\) is the surface measure on
\(\partial\mathcal{A}_H\). An indicator is not in the domain of the
Dirichlet form, so a ``restricted Poincaré constant'' over step
functions is vacuous. The right object is not the Dirichlet energy of
the indicator but the rate at which the semigroup relaxes it.

\begin{definition}[local relaxation rate]\label{def:local-rate}
For
measurable \(\mathcal{A}_H\) with \(0 < \pi(\mathcal{A}_H) < 1\), let
\(\phi_{\mathcal{A}_H} = \mathbf{1}_{\mathcal{A}_H} - \pi(\mathcal{A}_H) \in L^2(\pi)\),
which is centered and has
\(\|\phi_{\mathcal{A}_H}\|_{L^2(\pi)}^2 = \pi(\mathcal{A}_H)(1-\pi(\mathcal{A}_H))\).
The local relaxation rate of \(\mathcal{A}_H\) is the best exponential
rate at which the semigroup contracts this centered indicator,
\[
\lambda_{\mathcal{A}_H} \;:=\; \sup\big\{\, \lambda \ge 0 : \|P_t \phi_{\mathcal{A}_H}\|_{L^2(\pi)} \le \|\phi_{\mathcal{A}_H}\|_{L^2(\pi)}\, e^{-\lambda t} \ \text{for all } t \ge 0 \,\big\}.
\]
Equivalently, by the spectral theorem for the self-adjoint operator
\(-\mathcal{L}\), with
\(\mu_{\phi}(d\lambda) := d\langle E_\lambda \phi_{\mathcal{A}_H}, \phi_{\mathcal{A}_H}\rangle_\pi\)
the scalar spectral measure of \(\phi_{\mathcal{A}_H}\),
\[
\lambda_{\mathcal{A}_H} \;=\; \inf \operatorname{supp}(\mu_\phi).
\]
Because \(\phi_{\mathcal{A}_H} \perp \mathbf{1}\) and the Poincaré
inequality of \cref{thm:dynamic-mass} places the spectrum of
\(-\mathcal{L}\) on the nonconstant subspace in \([m, \infty)\), we have
\(\operatorname{supp}(\mu_\phi) \subseteq [m,\infty)\), hence
\[
\lambda_{\mathcal{A}_H} \;\ge\; m.
\]
The rate is exactly \(m\) when \(\phi_{\mathcal{A}_H}\) overlaps
the slowest mode of \(-\mathcal{L}\), and strictly larger when that
overlap vanishes. This is the spectral content of ``isolation'': an
isolated set is one whose indicator is orthogonal to the slow modes.
\end{definition}

\subsection*{Conductance as the geometric diagnostic} The quantity that
reads off isolation from the geometry of \(\mathcal{A}_H\) is the
conductance under the Langevin diffusion, with diffusion coefficient
\(\sigma^2/2\),
\[
h_{\mathcal{A}_H} \;:=\; \frac{\frac{\sigma^2}{2}\int_{\partial \mathcal{A}_H} \pi\, dS}{\min\big(\pi(\mathcal{A}_H),\, \pi(\mathcal{A}_H^c)\big)}.
\]
Say \(\mathcal{A}_H\) is {flux-isolated} if \(h_{\mathcal{A}_H}\)
is bounded below independently of \(d\). A small conductance forces a
slow-mode overlap and hence \(\lambda_{\mathcal{A}_H}\) close to \(m\),
while for the families of Section 6.3 the gain is verified directly: by
parity for symmetric sets (\cref{prop:tax-a}) and through an
Arrhenius-small slow-mode overlap for the far-tail ball
(\cref{prop:tax-b}). We do not assert a universal one-line Cheeger lower bound
\(\lambda_{\mathcal{A}_H} \ge h_{\mathcal{A}_H}^2/2\) for this
spectral-measure quantity; the conductance is used as the geometric
diagnostic, and the lower bound is established per family.

\subsection{The flux-isolation theorem}\label{sec:the-flux-isolation-theorem}

\begin{theorem}[mass control at the local rate]\label{thm:flux-isolation}
Assume \(m I \preceq \nabla^2 J\) and \(\chi_0^2 < \infty\). Then for
every measurable \(\mathcal{A}_H\) with \(0 < \pi(\mathcal{A}_H) < 1\)
and every \(t \ge 0\),
\[
\big|\nu_t(\mathcal{A}_H) - \pi(\mathcal{A}_H)\big| \;\le\; \sqrt{\chi_0^2\,\pi(\mathcal{A}_H)\,(1-\pi(\mathcal{A}_H))}\; e^{-\lambda_{\mathcal{A}_H} t} \;\le\; \sqrt{\chi_0^2\,\pi(\mathcal{A}_H)}\; e^{-\lambda_{\mathcal{A}_H} t},
\]
equivalently
\(\nu_t(\mathcal{A}_H) \le \pi(\mathcal{A}_H)\big(1 + \sqrt{\chi_0^2/\pi(\mathcal{A}_H)}\, e^{-\lambda_{\mathcal{A}_H} t}\big)\).
If in addition \(M := \|\nu_0/\pi\|_\infty < \infty\), then combining
with \cref{prop:max-principle},
\[
\nu_t(\mathcal{A}_H) \;\le\; \min\!\Big(M\,\pi(\mathcal{A}_H),\;\; \pi(\mathcal{A}_H) + \sqrt{\chi_0^2\,\pi(\mathcal{A}_H)}\; e^{-\lambda_{\mathcal{A}_H} t}\Big).
\]
\end{theorem}

This is \cref{thm:dynamic-mass} with the global rate \(m\)
replaced by the local rate \(\lambda_{\mathcal{A}_H} \ge m\). For a
flux-isolated set the burn-in time shrinks from
\(\frac{1}{2m}\log(\chi_0^2/\pi(\mathcal{A}_H))\) to
\(\frac{1}{2\lambda_{\mathcal{A}_H}}\log(\chi_0^2/\pi(\mathcal{A}_H))\),
a saving by the factor \(m/\lambda_{\mathcal{A}_H}\), and the \(\min\)
above caps the trajectory mass at \(M\pi(\mathcal{A}_H)\) for all \(t\),
so no swelling above that ceiling ever occurs.

\begin{proof}
Write
\(\phi := \phi_{\mathcal{A}_H} = \mathbf{1}_{\mathcal{A}_H} - \pi(\mathcal{A}_H)\),
so \(\|\phi\|_{L^2(\pi)}^2 = \pi(\mathcal{A}_H)(1-\pi(\mathcal{A}_H))\).
Since \(u_t - 1 = P_t(u_0 - 1)\) and \(\int(u_t - 1)\,d\pi = 0\),
centering the indicator and using self-adjointness of \(P_t\) in
\(L^2(\pi)\),
\[
\nu_t(\mathcal{A}_H) - \pi(\mathcal{A}_H) \;=\; \langle u_t - 1,\, \phi\rangle_\pi \;=\; \langle P_t(u_0 - 1),\, \phi\rangle_\pi \;=\; \langle u_0 - 1,\, P_t \phi\rangle_\pi.
\]
By the spectral representation in \cref{def:local-rate}, with
\(\operatorname{supp}(\mu_\phi) \subseteq [\lambda_{\mathcal{A}_H}, \infty)\),
\[
\|P_t \phi\|_{L^2(\pi)}^2 \;=\; \int e^{-2\lambda t}\, d\mu_\phi(\lambda) \;\le\; e^{-2\lambda_{\mathcal{A}_H} t}\int d\mu_\phi(\lambda) \;=\; e^{-2\lambda_{\mathcal{A}_H} t}\,\|\phi\|_{L^2(\pi)}^2.
\]
Cauchy-Schwarz, \(\|u_0 - 1\|_{L^2(\pi)} = \sqrt{\chi_0^2}\), and
\(\|\phi\|_{L^2(\pi)}^2 = \pi(\mathcal{A}_H)(1-\pi(\mathcal{A}_H)) \le \pi(\mathcal{A}_H)\)
give
\[
\big|\nu_t(\mathcal{A}_H) - \pi(\mathcal{A}_H)\big| \;\le\; \|u_0 - 1\|_{L^2(\pi)}\,\|P_t \phi\|_{L^2(\pi)} \;\le\; \sqrt{\chi_0^2\,\pi(\mathcal{A}_H)(1-\pi(\mathcal{A}_H))}\; e^{-\lambda_{\mathcal{A}_H} t}.
\]
Dividing the upper inequality by \(\pi(\mathcal{A}_H)\) gives the
multiplicative form. The \(\min\) statement adjoins the
maximum-principle ceiling of \cref{prop:max-principle}.
\end{proof}

\subsection*{Remark (why not the linear prefactor)} The qualitatively
stronger statement
\(|\nu_t(\mathcal{A}_H) - \pi(\mathcal{A}_H)| \le (M-1)\pi(\mathcal{A}_H)\,e^{-\lambda_{\mathcal{A}_H} t}\),
with prefactor linear in \(\pi(\mathcal{A}_H)\) and hence free
of any burn-in, would follow from an \(L^1(\pi)\) contraction
\(\|P_t \phi\|_{L^1(\pi)} \lesssim \|\phi\|_{L^1(\pi)} e^{-\lambda_{\mathcal{A}_H} t}\)
paired against \(\|u_0 - 1\|_\infty \le M - 1\). The semigroup is an
\(L^1(\pi)\) contraction but not at a positive rate in general;
obtaining the exponential \(L^1\) rate requires hypercontractivity
(available here from the Bakry-Émery log-Sobolev inequality) to lift
\(L^1\) to \(L^2\) after a fixed lag, which yields the global
log-Sobolev rate rather than the local rate \(\lambda_{\mathcal{A}_H}\).
We therefore state the rigorous \(L^2\) bound and obtain the
uniform-in-time control by intersecting it with the maximum-principle
ceiling, as in the \(\min\) above.

\subsection{High-dimensional taxonomy}\label{sec:high-dimensional-taxonomy}

The payoff of \(\lambda_{\mathcal{A}_H}\) is in the high-dimensional
geometry of a strongly log-concave \(\pi\), where the mass concentrates
on a thin shell. For the prototype \(J(Q) = \tfrac{1}{2}\|Q\|^2\) on
\(\mathbb{R}^d\), \(\pi\) concentrates on \(\|Q\| \approx \sqrt{d}\)
with \(O(1)\) width. Three families of small-\(\pi\)-mass failure
regions behave very differently.

\begin{figure}[t]
  \centering
  \footnotesize
  \includegraphics[width=\linewidth]{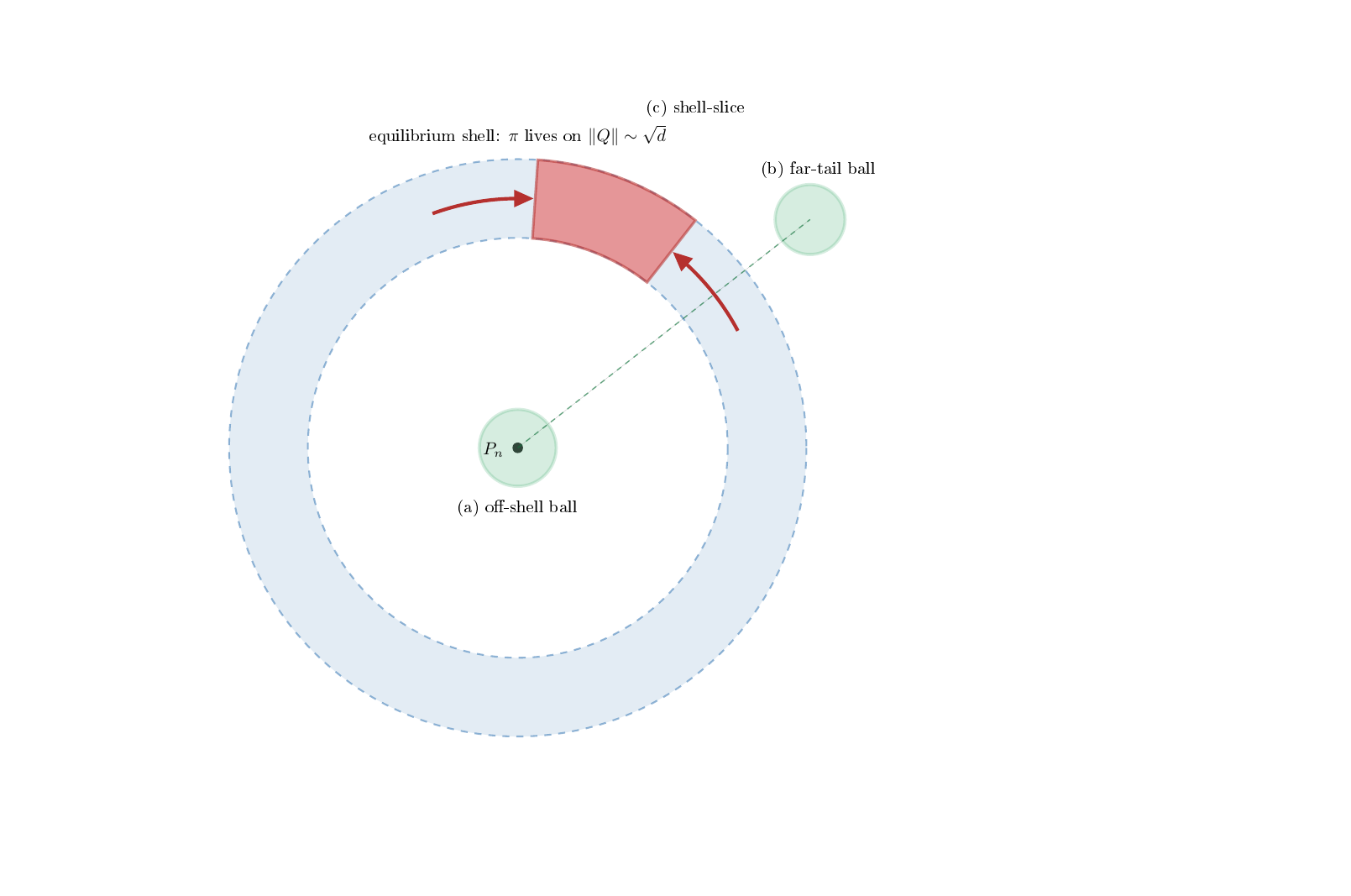}
  \caption{High-dimensional taxonomy of failure regions for the
    quadratic loss $J(Q) = \tfrac{1}{2}\|Q\|^2$, whose equilibrium
    $\pi$ concentrates on the shell $\|Q\| \sim \sqrt{d}$.
    (a) Off-shell ball at the mode: entropic moat, flux-isolated,
    $\lambda_{\mathcal{A}_H} = 2m$ (\cref{prop:tax-a}).
    (b) Far-tail ball: energetic moat (Arrhenius barrier), effective
    rate $2m$ up to exponentially long times (\cref{prop:tax-b}).
    (c) Shell-slice: mass flows freely along the shell,
    $\lambda_{\mathcal{A}_H} = m$ exactly (\cref{prop:tax-c}): no
    spectral gain, only the ceiling $M\pi(\mathcal{A}_H)$ survives.
    Shape-aware control applies to (a) via \cref{thm:flux-isolation}
    and to (b) via \cref{prop:tax-b}; for (c),
    \cref{thm:flux-isolation} reduces to \cref{thm:dynamic-mass}.}
  \label{fig:taxonomy}
\end{figure}

\begin{itemize}
\tightlist
\item
  \textbf{Off-shell ball at the mode}, \(\mathcal{A}_H = \{\|Q\| < r\}\)
  with \(r \ll \sqrt{d}\). The mode has lowest loss, but vanishing
  surface area faces the bulk: an entropic moat. Here
  \(\pi(\mathcal{A}_H) \sim e^{-d/2}\) while the boundary flux is
  comparably small, so \(h_{\mathcal{A}_H}\) is bounded below, and by
  parity \(\lambda_{\mathcal{A}_H} = 2m\) exactly (\cref{prop:tax-a}).
  \textbf{Flux-isolated; \cref{thm:flux-isolation} applies at the
  doubled rate.}
\item
  \textbf{Far-tail ball}, \(\mathcal{A}_H = \{\|Q - Q_0\| < r\}\) with
  \(\|Q_0\| \gg \sqrt{d}\). A genuine energy barrier of height
  \(\sim \|Q_0\|^2/2\): hitting times grow exponentially
  (Eyring-Kramers). The strict rate is \(\lambda_{\mathcal{A}_H} = m\),
  since any off-center set overlaps the slow linear modes, but the
  overlap is Arrhenius-small, so relaxation runs at the doubled rate
  \(2m\) throughout a window growing linearly in the barrier
  (\cref{prop:tax-b}). \textbf{Flux-isolated in the effective sense;
  the two-rate bound applies.}
\item
  \textbf{Thin shell-slice}, an angular wedge of the equilibrium shell
  \(\|Q\| \approx \sqrt{d}\). Equilibrium mass is small, but mass
  transports freely along the shell with no barrier: the slice overlaps
  the slow linear modes and \(\lambda_{\mathcal{A}_H} = m\) exactly
  (\cref{prop:tax-c}). \textbf{Not flux-isolated;
  \cref{thm:flux-isolation} reduces to \cref{thm:dynamic-mass}, and
  only the ceiling \(M\pi(\mathcal{A}_H)\) survives the transient.}
\end{itemize}

Small \(\pi(\mathcal{A}_H)\) is not enough to keep
\(\nu_t(\mathcal{A}_H)\) small for all \(t\); one also needs
\(\lambda_{\mathcal{A}_H}\) bounded below. The shell-slice is the
high-dimensional version of the swelling example, quantified in
Section 7.

\subsection{The taxonomy, proved for the quadratic prototype}\label{sec:taxonomy-proofs}

For the prototype the three cases can be settled exactly, because the
spectrum is explicit. Take \(J(Q) = \tfrac{m}{2}\|Q\|^2\), so
\(\pi = N(0, s^2 I)\) with \(s^2 = \sigma^2/(2m)\), and the normalized
eigenfunctions of \(-\mathcal{L}\) in \(L^2(\pi)\) are the Hermite
products
\(e_k(Q) = \prod_{i=1}^d \mathrm{He}_{k_i}(Q_i/s)/\sqrt{k_i!}\) over
multi-indices \(k \in \mathbb{N}_0^d\), with eigenvalues \(m|k|\),
\(|k| = k_1 + \cdots + k_d\) \cite[Ch.~4]{Pav14}. The spectral measure
\(\mu_\phi\) of the centered indicator is therefore atomic on
\(\{m, 2m, 3m, \ldots\}\), with mass
\(\sum_{|k| = j} \langle \phi_{\mathcal{A}_H}, e_k\rangle_\pi^2\) at
level \(mj\), and
\[
\lambda_{\mathcal{A}_H} \;=\; m \cdot \min\Big\{\, j \ge 1 \;:\; \textstyle\sum_{|k| = j} \langle \phi_{\mathcal{A}_H}, e_k\rangle_\pi^2 > 0 \,\Big\}.
\]
Two facts drive all three cases: \(e_k\) has parity \((-1)^{|k|}\), and
the level-one eigenfunctions are the linear coordinates \(Q_i/s\).

\begin{proposition}[case (a): symmetric sets and the off-shell ball]\label{prop:tax-a}
If \(\mathcal{A}_H = -\mathcal{A}_H\) with
\(0 < \pi(\mathcal{A}_H) < 1\), then
\(\lambda_{\mathcal{A}_H} \ge 2m\). For the centered ball
\(\mathcal{A}_H = \{\|Q\| < r\}\) with \(0 < r < \sqrt{d}\,s\), the rate
is exact: \(\lambda_{\mathcal{A}_H} = 2m\).
\end{proposition}

\begin{proof}
For \(k \neq 0\) the eigenfunctions are centered, so
\(\langle \phi_{\mathcal{A}_H}, e_k\rangle_\pi = \int_{\mathcal{A}_H} e_k\, d\pi\).
The substitution \(Q \mapsto -Q\) preserves \(\pi\) and
\(\mathcal{A}_H\) and multiplies \(e_k\) by \((-1)^{|k|}\), so every
odd-level coefficient vanishes and
\(\operatorname{supp}(\mu_\phi) \subseteq \{2m, 4m, \ldots\}\), giving
\(\lambda_{\mathcal{A}_H} \ge 2m\). For the centered ball, the radial
level-two eigenfunction is proportional to \(\|Q\|^2 - ds^2\), and on
\(\mathcal{A}_H\) the integrand is bounded above by \(r^2 - ds^2 < 0\),
so
\(\int_{\mathcal{A}_H} (\|Q\|^2 - ds^2)\, d\pi \le (r^2 - ds^2)\,\pi(\mathcal{A}_H) < 0\):
level two carries positive mass and \(\lambda_{\mathcal{A}_H} = 2m\).
\end{proof}

\begin{proposition}[case (b): far-tail ball, two-rate bound]\label{prop:tax-b}
Let \(\mathcal{A}_H = \{\|Q - Q_0\| \le r\}\) with
\(\rho := \|Q_0\| - r > 0\), and let
\[
\theta^2 \;:=\; \frac{\sum_{i=1}^d \langle \phi_{\mathcal{A}_H},\, Q_i/s\rangle_\pi^2}{\pi(\mathcal{A}_H)\,\big(1 - \pi(\mathcal{A}_H)\big)} \;\in\; [0, 1]
\]
be the fraction of the indicator's variance carried by the slow
(level-one) modes. Then for all \(t \ge 0\),
\[
\|P_t\, \phi_{\mathcal{A}_H}\|_{L^2(\pi)}^2 \;\le\; \pi(\mathcal{A}_H)\,\big(\theta^2 e^{-2mt} + e^{-4mt}\big),
\]
and hence, for any \(\nu_0\) with \(\chi_0^2 < \infty\),
\[
\big|\nu_t(\mathcal{A}_H) - \pi(\mathcal{A}_H)\big| \;\le\; \sqrt{\chi_0^2\,\pi(\mathcal{A}_H)}\;\big(\theta\, e^{-mt} + e^{-2mt}\big).
\]
The slow-mode fraction is Arrhenius-small:
\[
\theta \;\le\; \frac{\|Q_0\| + r}{s}\,\sqrt{\frac{\pi(\mathcal{A}_H)}{1 - \pi(\mathcal{A}_H)}}, \qquad \pi(\mathcal{A}_H) \;\le\; e^{-\rho^2/(2s^2)}.
\]
In particular the strict rate is \(\lambda_{\mathcal{A}_H} = m\) (an
off-center set overlaps the linear modes), but
\(\|P_t\, \phi_{\mathcal{A}_H}\|_{L^2(\pi)} \le \sqrt{2\,\pi(\mathcal{A}_H)}\; e^{-2mt}\)
for all \(t \le \frac{1}{m}\log(1/\theta)\): relaxation runs at the
doubled rate throughout a window whose length grows linearly in the
Arrhenius exponent \(\rho^2/(4s^2) = m\rho^2/(2\sigma^2)\).
\end{proposition}

\begin{proof}
Split \(\mu_\phi\) into its mass at level one, which is
\(\theta^2\,\|\phi_{\mathcal{A}_H}\|_{L^2(\pi)}^2\), and the rest,
supported in \([2m, \infty)\) with total mass at most
\(\|\phi_{\mathcal{A}_H}\|_{L^2(\pi)}^2 \le \pi(\mathcal{A}_H)\). Then
\[
\|P_t\, \phi_{\mathcal{A}_H}\|_{L^2(\pi)}^2 \;=\; \int e^{-2\lambda t}\, d\mu_\phi(\lambda) \;\le\; \pi(\mathcal{A}_H)\,\big(\theta^2 e^{-2mt} + e^{-4mt}\big),
\]
and the trajectory bound follows by pairing against \(u_0 - 1\) as in
the proof of \cref{thm:flux-isolation}, using
\(\sqrt{a + b} \le \sqrt{a} + \sqrt{b}\). For the overlap, rotate
coordinates so that \(Q_0 = \|Q_0\|\, e_1\). Since
\(\langle \phi_{\mathcal{A}_H}, Q_i/s\rangle_\pi = \frac{1}{s}\int_{\mathcal{A}_H} Q_i\, d\pi\),
\[
\sum_{i=1}^d \langle \phi_{\mathcal{A}_H}, Q_i/s\rangle_\pi^2 \;=\; \frac{1}{s^2}\,\Big\|\int_{\mathcal{A}_H} Q\, d\pi\Big\|^2 \;\le\; \frac{(\|Q_0\| + r)^2}{s^2}\,\pi(\mathcal{A}_H)^2,
\]
because \(\|Q\| \le \|Q_0\| + r\) on \(\mathcal{A}_H\); dividing by
\(\pi(\mathcal{A}_H)(1 - \pi(\mathcal{A}_H))\) gives the bound on
\(\theta^2\). For the mass, \(\mathcal{A}_H\) lies in the half-space
\(\{Q_1 \ge \rho\}\) and \(Q_1 \sim N(0, s^2)\) under \(\pi\), so the
Gaussian tail bound gives
\(\pi(\mathcal{A}_H) \le e^{-\rho^2/(2s^2)}\). Finally
\(\int_{\mathcal{A}_H} Q_1\, d\pi \ge \rho\,\pi(\mathcal{A}_H) > 0\), so
level one carries positive mass and \(\lambda_{\mathcal{A}_H} = m\).
For the last claim, \(\theta^2 e^{-2mt} \le e^{-4mt}\) if and only if
\(t \le \frac{1}{m}\log(1/\theta)\), and on that window the two-term
bound is at most \(2\,\pi(\mathcal{A}_H)\, e^{-4mt}\).
\end{proof}

The two-rate estimate itself uses only
\(\sup_{\mathcal{A}_H}\|Q\| \le \|Q_0\| + r\) and so holds for any
bounded set of that radius; what is specific to the far-tail ball is
that its position buys the Arrhenius-small \(\theta\) through the mass
bound.

\begin{proposition}[case (c): shell-slice]\label{prop:tax-c}
Let
\(\mathcal{A}_H = \{Q : R_1 \le \|Q\| \le R_2,\ \langle Q, e_1\rangle \ge \|Q\|\cos\alpha\}\)
with \(0 < R_1 \le R_2 \le \infty\) and \(\alpha \in (0, \pi/2)\). Then
\[
\langle \phi_{\mathcal{A}_H},\, Q_1/s\rangle_\pi \;\ge\; \frac{R_1\cos\alpha}{s}\;\pi(\mathcal{A}_H) \;>\; 0,
\]
so level one carries positive mass and
\(\lambda_{\mathcal{A}_H} = m\) exactly: for the shell-slice,
\cref{thm:flux-isolation} reduces to \cref{thm:dynamic-mass} and the
shape-aware route yields no improvement. The transient overshoot this
permits is realized, quantitatively, in Section 7.
\end{proposition}

\begin{proof}
On \(\mathcal{A}_H\),
\(Q_1 = \langle Q, e_1\rangle \ge \|Q\|\cos\alpha \ge R_1\cos\alpha > 0\),
so
\(\langle \phi_{\mathcal{A}_H}, Q_1/s\rangle_\pi = \frac{1}{s}\int_{\mathcal{A}_H} Q_1\, d\pi \ge \frac{R_1\cos\alpha}{s}\,\pi(\mathcal{A}_H) > 0\).
Level one carries positive mass, so \(m \in \operatorname{supp}(\mu_\phi)\)
and, with \(\lambda_{\mathcal{A}_H} \ge m\) from \cref{def:local-rate},
\(\lambda_{\mathcal{A}_H} = m\).
\end{proof}

\subsection*{Remark (what the propositions change)} The propositions
sharpen the informal labels. The off-shell ball is isolated in the
strict spectral sense: \(\lambda_{\mathcal{A}_H} = 2m\). The far-tail
ball is not: any off-center set overlaps the slow linear modes, so its
strict rate is exactly \(m\), and \cref{def:local-rate} alone records
no gain. The gain is real but takes the two-rate form of
\cref{prop:tax-b}: the slow mode enters with an Arrhenius-small
prefactor, so the doubled rate governs until times growing linearly in
the barrier. The shell-slice also has strict rate \(m\), with a
slow-mode coefficient bounded below at the scale of its mass
(\cref{prop:tax-c}), and no spectral mechanism improves on
\cref{thm:dynamic-mass} for it: its uniform-in-time control comes only
from the ceiling of \cref{prop:max-principle} and the angular bounds of
Section 7. The upshot for \cref{def:local-rate} is that the strict
infimum \(\lambda_{\mathcal{A}_H}\) is brittle, collapsing to \(m\)
under a vanishingly small slow-mode overlap, and the two-rate bound of
\cref{prop:tax-b} is its robust refinement.

\section{Sharpness and worked example: Ornstein-Uhlenbeck and the shell}\label{sec:sharpness-and-worked-example-ornstein-uhlenbeck-and-the-shell}

This section makes the swelling phenomenon quantitative for the exactly
solvable Ornstein-Uhlenbeck (OU) case. It shows that the burn-in window
of Section 5 and the flux-isolation hypothesis of Section 6 are both
necessary, and it exhibits the precise geometry, an angular slice of the
equilibrium shell, where the transient overshoot is exponential in
\(d\).

Take the quadratic loss \(J(Q) = \tfrac{1}{2}\sum_i \lambda_i Q_i^2\)
with \(0 < \lambda_1 \le \cdots \le \lambda_d\), so the dynamics
decouples into coordinatewise OU processes and \(\pi = N(0, \Sigma)\)
with \(\Sigma = \operatorname{diag}(\sigma^2/(2\lambda_i))\). Start from
an isotropic Gaussian \(\nu_0 = N(0, \sigma_0^2 I)\). Two quantities
organize the analysis: the \textbf{transient overshoot factor}
\[
C_{\mathrm{bad}} \;:=\; \frac{\nu_{t_\star}(\mathcal{A}_H)}{\nu_0(\mathcal{A}_H)}
\]
at the worst time \(t_\star\), and the absolute mass
\(\nu_t(\mathcal{A}_H)\) itself.

\subsection{Isotropic case}\label{sec:isotropic-case}

Let \(J(Q) = \tfrac{a}{2}\|Q\|^2\) and let \(\mathcal{A}_H\) be confined
to a level-set shell at radius \(B^\star\). With the single
dimensionless ratio
\[
u \;:=\; \frac{(B^\star)^2}{d\,\sigma_0^2} \qquad \text{(squared shell radius over squared initial typical radius),}
\]
a direct Gaussian computation gives
\[
C_{\mathrm{bad}} \;=\; \exp\!\left[\frac{d}{2}\,\phi(u)\right], \qquad \phi(u) := u - 1 - \log u \;\ge\; 0,
\]
with equality \(\phi(u) = 0\) only at \(u = 1\). So the overshoot is
exponential in \(d\) at rate \(\tfrac12\phi(u)\) whenever the shell
radius is mismatched to the initial scale. This is the high-dimensional
swelling: a set of tiny mass can gain an \(e^{\Theta(d)}\) factor in
transit. The saving grace in the isotropic case is that \(\nu_t\) stays
isotropic, so its mass on any shell set is bounded by the angular
fraction of the shell,
\[
\nu_t(\mathcal{A}_H) \;\le\; R_{\mathrm{shell}} \qquad \text{(angular fraction of $\mathcal{A}_H$ on $S^{d-1}$).}
\]
If \(R_{\mathrm{shell}}\) is itself exponentially small in \(d\), the
overshoot is absorbed before it reaches absolute mass. This is
the dichotomy of Section 6: the shell-slice has small
\(\pi(\mathcal{A}_H)\) (equilibrium reason) but is not flux-isolated
(transit reason), so the burn-in window is real and \(C_{\mathrm{bad}}\)
measures its depth.

\subsection{Anisotropic case}\label{sec:anisotropic-case}

With distinct \(\lambda_i\), the coordinates run at rates
\(2\lambda_i\). Writing the per-coordinate ratios
\(u_i = Q_{0,i}^2/\sigma_0^2\), the overshoot
factors out:
\[
C_{\mathrm{bad}} \;\le\; \exp\!\left[\frac{1}{2}\sum_{i=1}^d \phi(u_i)\right].
\]
Equality requires the per-coordinate peak times to align, which
generically fails once the \(\lambda_i\) spread, so the
\textbf{effective dimension}
\(d_{\mathrm{eff}} = \#\{i : u_i \text{ far from } 1\}\) is governed by
the geometry of \(\mathcal{A}_H\) relative to the initial scale.
Anisotropy also destroys the uniform angular density, weakening the
angular bound to
\[
\nu_t(\mathcal{A}_H) \;\le\; R_{\mathrm{shell}} \cdot \prod_{i=1}^d \frac{\sigma_{\max}(t)}{\sigma_i(t)},
\]
where \(\sigma_i(t)\) is the time-\(t\) standard deviation in coordinate
\(i\). The prefactor is \(1\) at \(t = 0\) and grows to
\(\prod_i \sqrt{\lambda_i/\lambda_1}\) as \(t \to \infty\), exponential
in \(d\) for a spread spectrum.

\subsection{Bounded condition number}\label{sec:bounded-condition-number}

The clean controllable case is a bounded condition number
\(\kappa := \lambda_d/\lambda_1\) together with an initial scale chosen in
range, \(\sigma_0^2 \in [1/\lambda_d,\, 1/\lambda_1]\) (in the
\(\sigma = \sqrt{2}\) normalization). Then \(\sigma_i^2(t)\) stays in
\([1/\lambda_d, 1/\lambda_1]\) for all \(i\) and \(t\), so
\(\sigma_{\max}/\sigma_i \le \sqrt{\kappa}\) and the angular bound becomes
uniform in time,
\[
\nu_t(\mathcal{A}_H) \;\le\; R_{\mathrm{shell}} \cdot \kappa^{d/2} \qquad \text{for all } t \ge 0.
\]
Three regimes: \(\kappa = 1\) recovers the isotropic bound;
\(\kappa - 1 = O(1/d)\) keeps \(\kappa^{d/2} = O(1)\); a fixed \(\kappa > 1\) gives
exponential growth at rate \(\tfrac12\log \kappa\). The bound is useful when
\(R_{\mathrm{shell}}\) decays faster than \(\kappa^{d/2}\) grows, that is for
\(\mathcal{A}_H\) of angular radius below \(\arcsin(1/\sqrt{\kappa})\) on
\(S^{d-1}\).

\subsection{Combined statement and where dimension enters}\label{sec:combined-statement-and-where-dimension-enters}

Putting the transit and angular bounds together,
\[
\nu_t(\mathcal{A}_H) \;\le\; \min\!\Big(\, C_{\mathrm{bad}}\cdot \nu_0(\mathcal{A}_H),\;\; R_{\mathrm{shell}}\cdot \textstyle\prod_i \sigma_{\max}(t)/\sigma_i(t)\,\Big).
\]
Two independent sources of exponential-in-\(d\) growth appear: the
overshoot \(C_{\mathrm{bad}} = \exp[\tfrac12\sum_i \phi(u_i)]\),
controlled by the geometry of \(\mathcal{A}_H\) through the \(u_i\); and
the angular amplification
\(\prod_i \sigma_{\max}/\sigma_i \le \kappa^{d/2}\), controlled by the
spectral spread of \(J\). Avoiding both requires either placing
\(\mathcal{A}_H\) in a low-\(d_{\mathrm{eff}}\) subspace with the
initial law matched to stationarity in the complementary directions, or
a bounded \(\kappa\) with \(\sigma_0\) in range and \(\mathcal{A}_H\)
angularly thin. This is the constructive counterpart of the
flux-isolation hypothesis of \cref{thm:flux-isolation}.

\subsection{The one-dimensional swelling example}\label{sec:the-one-dimensional-swelling-example}

The simplest instance, recovering the introduction: \(J(Q) = Q^2/2\) on
\(\mathbb{R}\) with \(\sigma^2 = 2\), so \(\pi = N(0,1)\) and \(m = 1\).
Start at \(\nu_0 = \delta_{10}\) and take \(\mathcal{A}_H = [4,6]\).
Then \(\nu_0(\mathcal{A}_H) = 0\),
\(\pi(\mathcal{A}_H) \approx 10^{-5}\), but
\(\nu_t = N(10 e^{-t}, 1 - e^{-2t})\), so at \(t = \log 2\) the law is
\(N(5, 3/4)\) and \(\nu_t(\mathcal{A}_H) \approx 0.7\), an overshoot of
about \(10^5\). Here \(M = \infty\) (point mass), so the ceiling
\cref{prop:max-principle} is vacuous, and \(\mathcal{A}_H\) is
the one-dimensional analogue of a shell-slice on the transport path, so
\(\lambda_{\mathcal{A}_H}\) is not bounded below and \cref{thm:flux-isolation} does not apply. Both protective hypotheses fail at
once, which is why the example is canonical.

\section{Summary}\label{sec:summary}

At stationarity, the failure region has
exponentially-small-in-\(d\) mass,
\(\pi(\mathcal{A}_H) \le K e^{-Ad/2}\) (\cref{thm:static-mass}),
with a complementary Arrhenius rate in the small-noise regime. Along the
trajectory, a shape-free bound (\cref{thm:dynamic-mass}) relaxes
to twice this static mass after a burn-in of order \(d\), using only the
total mass and the global gap \(m\); the conversion step transfers to
discrete-time Langevin Monte Carlo (\cref{sec:discrete-time}). Adding
geometric information, a local relaxation rate
\(\lambda_{\mathcal{A}_H} \ge m\)
(\cref{def:local-rate}, \cref{thm:flux-isolation}) replaces the global rate
by a faster set-dependent one for flux-isolated sets, shortening the
burn-in by the factor \(m/\lambda_{\mathcal{A}_H}\); combined with the
maximum-principle ceiling it caps the trajectory mass uniformly in time.
The three-case taxonomy behind this dichotomy is proved exactly for the
quadratic prototype (\cref{sec:taxonomy-proofs}).

The two dynamic bounds are the two ends of one tradeoff.
Total-mass control is universal and dimension-free but transiently
uninformative; shape-aware control shortens the transient and, via the
ceiling, caps the mass uniformly, but requires verifying that
\(\mathcal{A}_H\) is geometrically isolated and that the start has a
bounded density ratio. The worked OU example (Section 7) shows the gap
between them is real: an angular shell-slice has tiny equilibrium mass
yet an \(e^{\Theta(d)}\) transient overshoot, and only the
flux-isolation hypothesis rules it out.

\section*{Acknowledgments}
The author used an AI assistant (Anthropic's Claude) in preparing this
manuscript, for help with drafting and editing the exposition and with
organizing and checking the mathematical arguments. The author has
personally verified all results, proofs, numerical claims, and
references, and takes full responsibility for the content, consistent
with SIAM's policy on the use of AI tools.

\bibliographystyle{siamplain}
\bibliography{references}

\end{document}